\documentclass[runningheads]{llncs}
\usepackage[T1]{fontenc}
\newcommand*\foobar{\includegraphics[width=16pt]{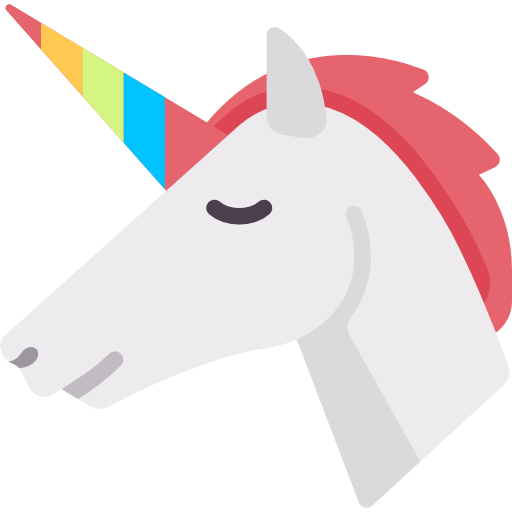}}
 \usepackage{kotex} 
\newcommand{\starcup}{$\sqcup$\kern-0.58em{$\star$}}
\usepackage{graphicx}
\usepackage{amsmath}
\usepackage{amssymb}
\usepackage{booktabs}
\usepackage{mathtools}
\usepackage[linesnumbered,ruled,vlined]{algorithm2e}
\usepackage{latexsym}
\usepackage{graphicx}
\usepackage{multirow}
\usepackage{adjustbox}
\usepackage{wrapfig}
\usepackage{makecell}
\usepackage{subcaption}
\usepackage{xcolor}
\usepackage{caption}
\usepackage{graphicx}

\usepackage{hyperref}

\usepackage[capitalize]{cleveref}
\crefname{section}{Sec.}{Secs.}
\Crefname{section}{Section}{Sections}
\Crefname{table}{Table}{Tables}
\crefname{table}{Tab.}{Tabs.}
\newtheorem{prop}{Proposition}

\begin{document}
\title{\foobar \,UNICORN: Ultrasound Nakagami Imaging via Score Matching and Adaptation}
\titlerunning{\foobar \,UNICORN}
%

\author{Kwanyoung Kim$^{*1}$, \; Jaa-Yeon Lee$^{*1}$, \; Jong Chul Ye$^{1}$,} 
\institute{Kim JaeChul Graduate School of AI, KAIST, Daejun, Repubilc of Korea$^{1}$
\email{\tt\small \{cubeyoung, jaayeon, jong.ye\}@kaist.ac.k } }

\authorrunning{Kwanyoung Kim et.al}

%
%
\maketitle              
\begin{abstract}

Nakagami imaging holds promise for visualizing and quantifying tissue scattering in ultrasound waves, with potential applications in tumor diagnosis and fat fraction estimation which are challenging to discern by  conventional ultrasound B-mode images. Existing methods struggle with optimal window size selection and suffer from estimator instability, leading to degraded resolution images. To address this, here we propose a novel method called UNICORN (\textbf{U}ltrasound \textbf{N}akagami \textbf{I}maging via S\textbf{cor}e Matching and Adaptatio\textbf{n}), that offers an accurate, closed-form  estimator for Nakagami parameter estimation in terms of the score function of ultrasonic envelope. 
Extensive experiments using simulation and real ultrasound RF data demonstrate UNICORN's superiority over conventional approaches in accuracy and resolution quality.

\end{abstract}
\section{Introduction}
\def\thefootnote{*}\footnotetext{These authors contributed equally to this work}

Ultrasound imaging is an essential tool for providing real-time qualitative descriptions of tissue morphology.
While ultrasound B-mode imaging offers qualitative information about tissue properties, its ability to provide precise tissue characterization for clinical decision-making is often limited. 

To overcome this limitation, quantitative ultrasound (QUS) methods have emerged. By analyzing statistical features of ultrasound radiofrequency (RF) echo signals, including tissue backscatter~\cite{oelze2016review} and ultrasonic wave attenuation~\cite{bigelow2005estimation}, QUS enables finer distinctions among tissue types beyond the capabilities of B-mode imaging alone. 
Since backscattered US signals encapsulate the characteristics of scatterers within tissue, such as their shape, size, density, and other properties~\cite{bamber1981acoustic,insana1990describing}, QUS methods utilizing these signals have been explored to visualize ultrasound wave scatter properties within tissue~\cite{tsui2007imaging}. 

Initially, the Rayleigh distribution was employed to model the backscattered signal data. However, scatterers in tissue exhibit randomly varied scattering patterns ranging from pre-Rayleigh to post-Rayleigh distribution~\cite{shankar1995model,dutt1994ultrasound}. Therefore, to better explain the backscattering behaviors and cover these variants of distribution, the Nakagami distribution has been studied as a general statistical model since the corresponding Nakagami parameter estimated from the backscattered echoes can be used to identify various backscattering distributions in medical ultrasound~\cite{zhang2012feasibility}. 
It has been shown that using Nakagami parameter mapping assists in detecting and characterizing abnormalities, such as those in the breast, liver, and kidney~\cite{shankar2000general}. Furthermore, it helps facilitate the detection of scatterer concentrations and arrangements, enhancing tumor characterization, and improving the classification of benign and malignant breast tumors~\cite{liao2011classification,tsui2010ultrasonic}. 

Previous Nakagami imaging studies in US commonly utilize moment-based and maximum-likelihood (ML) estimators for Nakagami distribution parameter mapping. In the moment-based approach, Nakagami parameters are computed within a sliding local window, typically optimized at three times the transducer pulse length for optimal window size~\cite{tsui2007imaging}. In contrast, ML estimator-based Nakagami imaging yields more consistent results with smaller variances~\cite{cheng2001maximum}. Window-modulated compounding (WMC) Nakagami imaging employs a moment estimator with a range of local window sizes, enhancing smoothness~\cite{tsui2014window}. However, both moment and ML estimators rely on the sliding window technique, necessitating a trade-off between image resolution and estimator stability~\cite{larrue2011nakagami}. Increasing window size improves smoothness but reduces resolution, while smaller sizes offer finer resolution at the expense of stability.

%

To address this, here we introduce a novel  framework for \textbf{U}ltrasound \textbf{N}akagami \textbf{i}maging via S\textbf{cor}e Matching and Adapatio\textbf{n}, termed UNICORN. 
Inspired by the success of self-supervised denoising approach that utilizes  the score function, i.e. the gradient of loglikelihood \cite{kim2021noise2score,kim2022noise},
this framework offers a closed-form solution for mapping Nakagami parameters in terms of the {\em score function} of the measurement. Specifically, by integrating the score function of RF envelope signals, UNICORN directly computes Nakagami images, eliminating the need for the sliding window technique. Consequently, our proposed technique preserves ultrasound imaging resolution in Nakagami imaging while ensuring stability. Our contributions can be summarized as:

\begin{itemize}
\item[-] We provide a novel closed-form solution for estimating the Nakagami parameter in terms of the score function of RF envelope data, which completely overcomes the limitations of conventional methods.
\item[-] Through extensive experiments including simulation and real RF envelope data, our proposed method, called UNICORN,  demonstrates its superiority in estimation performance and the classification of benign and malignant breast tumor. 
\end{itemize}

\begin{figure}[t]
  \begin{center}
    \includegraphics[width=1\textwidth]{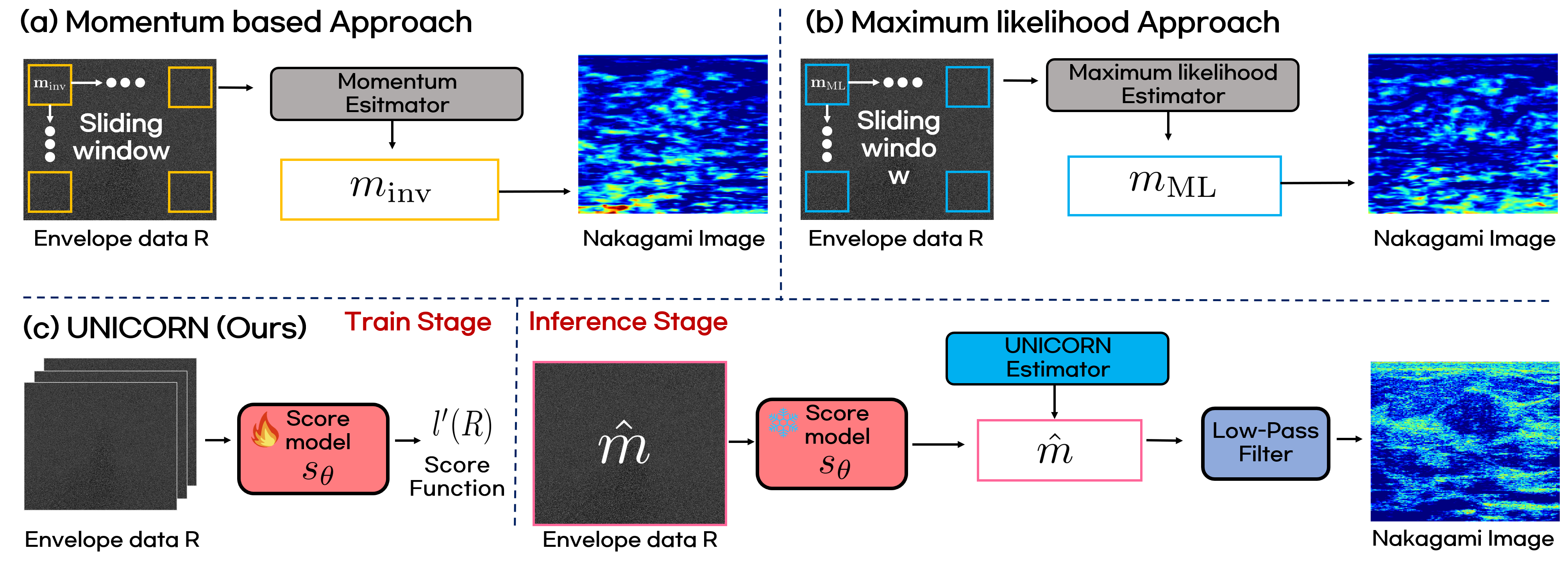}
     \vspace{-1em}
  \end{center}
  \caption{
Nakagami imaging using conventional methods and our UNICORN framework.
(a) Momentum-based approach uses a sliding window for Nakagami parameter calculation.
(b) Maximum likelihood method obtains Nakagami image through ML estimation.
(c) UNICORN consists of two stages: training a score model to learn RF envelope score function, and inference step estimates Nakagami image in terms of score function. }
  \vspace{-1em}
  \label{fig:main}
\end{figure}

\section{Background}

The probability density function of the ultrasound backscattered envelope $R$ under the Nakagami statistical model can be described as follows:
\begin{align}
    p_{R}(r) = \frac{2}{\Gamma(m)}\left(\frac{m}{\Omega}\right)^{m}r^{2m-1}\exp\left(-\frac{m}{\Omega}r^2\right)\mathcal{U}(r) 
    \label{pd:naka}
\end{align}
where $r$ is a value of the random variable $R$,
$\Gamma(\cdot)$ and $\mathcal{U}(\cdot)$ represent the Gamma function and the unit step function, respectively. $\Omega$ denotes the second moment, \textit{i.e.} $\Omega$ = $\mathbb{E}[R^{2}]$, while $m$ is the Nakagami parameter that determines the statistical distribution of the ultrasound backscattered envelope. Specifically, a Nakagami parameter ranging from 0 to 1 indicates a transition from a pre-Rayleigh to a Rayleigh distribution, while a parameter larger than 1 implies that the statistics of the backscattered signal conform to post-Rayleigh distributions.
The Nakagami parameter $m$ with momentum estimator can be calculated given by:
\begin{align}
    m_{\text{inv}} = \frac{[\mathbb{E}(R^2)]^2}{\mathbb{E}[R^2 - \mathbb{E}(R^2)]^2} \label{m_estimator}
\end{align}
where $m_{\text{inv}}$ denotes the esimated Nakagmi paramter by the momentum estimator \cite{tsui2007imaging,tsui2010ultrasonic,tsui2014window}. 
On the other hand, in the maximum likelihood (ML) estimator approach \cite{cheng2001maximum}, $m_{\text{ML}}$ is determined by maximizing the likelihood function of \cref{pd:naka} with respect to $m$. Both estimators calculate the local Nakagami parameter within a square window of a specific size and assign it to the new pixel located at the center of the window. This process is repeated until the window covers the entire envelope image by sliding across it, as illustrated in \cref{fig:main} (a) and (b).

\section{Methods}

Recent works of Noise2Score \cite{kim2021noise2score,kim2022noise} provided a highly
efficient closed form formula for self-supervised image denoising using Tweedies' formula that utilizes the score function, i.e. the gradient
of loglikelihood of the measurement.
Inspired by this, here we introduce a novel framework for Ultrasound Nakagami Imaging with Score Matching and Adaptation, termed UNICORN. UNICORN consists of two steps: firstly, we learn the score function of the ultrasonic envelope data via denoising score matching loss. Then, we estimate the Nakagami parameter $m$ per pixel, followed by a low-pass filter, and reconstruct the Nakagami Imaging, as illustrated in \cref{fig:main}.

\subsection{Nakagami parameter Estimator with Score function} \label{sec:estimator}
Instead of directly maximizing the likelihood function with respect to $m$ in the existing ML estimator~\cite{cheng2001maximum}, which entails complex computation and approximation errors, we adopt to maximize the likelihood with respect to envelope data, $R$, and can obtain the following the closed-form estimator.
\begin{prop}
For the the given measurement model \eqref{pd:naka},
the estimate of the unknown Nakagami parameter $m$ is given by
\begin{align}
\hat m= \frac{\frac{1}{r} + \nabla_r \log p_{R}(r) }{\left(\frac{2}{r} - \frac{2r}{\hat \Omega}\right)}, \quad \text{where} \quad \hat \Omega = \mathbb{E}[R^2] .
\end{align}
where 
$\nabla_r \log p_{R}(r)$ is the score function of the RF envelope data $R$. 
\label{prop1}
\end{prop}
\begin{proof}
See \cref{proof1}.
\end{proof}

Notably, in the derivation of \cref{prop1}, no approximation is necessary, resulting in more accurate estimations compared to existing approaches. To enhance the algorithm's robustness against outliers, we also incorporate a low-pass filter, such as a median or average filter. Subsequently, we estimate the final Nakagami image, denoted as $m_{\text{UNICORN}}$, as follows:
\begin{align}
m_{\text{UNICORN}} = \texttt{low-pass Filter} (\hat m)
\end{align}


\subsection{Loss function for Denoising Score Matching} \label{sec:loss}

To learn the score function from the ultrasound backscattered envelope $R$, 
we employ the amortized residual DAE (AR-DAE), which is a stabilized implementation of denoising autoencoder \cite{vincent2010stacked}.
Specifically,  AR-DAE loss function is defined by:
\begin{align}
\underset{\Theta}{\arg \min }  ={\underset{R \sim p_{R}(r)}{\mathbb{E}}}\|u + \sigma_a s_\theta(R + \sigma_a u)\|^2
\label{loss:ar-dae}
\end{align}
where $s_{\theta}$ is the score model parameterized by $\theta$, $u \sim \mathcal{N}(0,I)$, and $\sigma_a$ $\sim \mathcal{N}(0,\delta^2)$. $\sigma_a$ is perturbed noise which gradually decreases with an annealing schedule. Minimizing Eq. (\ref{loss:ar-dae}) provides the network $s_{\theta^{\ast}}$ which can directly estimate the score function of envelope data, $s_{\theta^{\ast}} = \nabla_{R}\log p_{R}(r) = l'(r)$.
Estimating the score function of measurements via \cref{loss:ar-dae} has been demonstrated to be both direct and stable \cite{kim2021noise2score,kim2022noise}. Therefore, we adopt this approach as the initial step of our method (see \cref{fig:main} (c)).
 
\vspace{-1em}
\section{Experiments}
In this work, we conduct two categories of experiments to validate the effectiveness of our proposed method. Firstly, we perform simulations using synthetic Nakagami distributions on grayscale image datasets and ultrasound image dataset. Secondly, we apply our method to real ultrasound RF envelope datasets to validate the effectiveness of breast tumor classification.
\vspace{-1em}
\subsection{Implementation Details}
To ensure a fair comparison with other methods, we compare our proposed method with other conventional approaches such as momentum estimator~\cite{tsui2007imaging}, window-modulated compounding estimator~\cite{tsui2014window} and maximum likelihood approach~\cite{cheng2001maximum}. The detail of the model and hyper-parameter is available in \cref{implement}.

\vspace{-1em}
\subsection{Datasets}
\textbf{Synthetic Simulation Experiments} We conduct evaluations on both the MNIST and the BUSI ultrasound image datasets. For the MNIST dataset, we train our neural network to learn the score function using the training set and evaluate its performance on the test set.
For the BUSI dataset~\cite{al2020dataset}, which consists of breast ultrasound B-mode images from women
we adopt it for our ultrasound image dataset. 
For simulation experiments, we initially normalize images from the range of 0 to 1 to the range of 0.5 to 2 and set it as the ground truth $m$. Subsequently, we apply the Nakagami distribution (as described in \cref{pd:naka}) per pixel to generate synthetic measurement with $\Omega = 1$.

\noindent\textbf{Real RF Envelope Experiment} For this setting, we utilize the OASBUD dataset, which comprises 100 ultrasound (US) images corresponding to 48 benign and 52 malignant breast masses, with one mass per image~\cite{piotrzkowska2017open}. 
The dataset includes the backscattered signal and corresponding mask indicating the presence of a tumor. 
The detail of the dataset is available in \cref{implement}.

\subsection{Evaluation Metrics}
\textbf{Synthetic Simulation Experiments} Since we generate the synthetic measurement under Nakagmi distribution by using ground truth image, we evaluate the performance of simulation results using PSNR (Peak Signal-to-Noise Ratio) and RMSE (Root Mean Square Error) metrics across various methods. 

\noindent\textbf{Real RF Envelope Experiment}\label{sec:eval}
In this experiment, we are unable to access the ground truth data for Nakagami imaging, preventing quantitative evaluation against a reference. Nonetheless, prior research suggests that both benign and malignant tumor tissue conforms to a pre-Rayleigh distribution, characterized by Nakagami parameters smaller than 1. Furthermore, malignant tumors tend to exhibit a more evident pre-Rayleigh distribution in their backscattered signals compared to benign tumors~\cite{tsui2010classification,larrue2014modeling}. Therefore, we assess qualitative results based on the closeness of Nakagami parameters to 0, indicating distinct pre-Rayleigh distributed outcomes. Additionally, we conduct an analysis of the differences between benign and malignant tumors by comparing the average and standard deviation obtained from each approach.

\begin{table}[!t]
	\centering
	\caption{Quantitative simulation results for MNIST dataset and ultrasound image BUSI dataset using various methods. The \textbf{bold} numbers indicate the best performance. WS denotes sliding window size.}
	\resizebox{0.95\linewidth}{!}{
		\begin{tabular}{ccccccc}
			\toprule
			Dataset & \multicolumn{3}{c}{MNIST} & \multicolumn{3}{c}{BUSI} \\ 
            \cmidrule(r){1-1} \cmidrule(r){2-4} \cmidrule(r){5-7} 
			Metric & WS & PSNR (dB) $\uparrow$ & RMSE $\downarrow$ & WS & PSNR (dB) $\uparrow$ & RMSE $\downarrow$ \\ 
			\cmidrule(r){1-1} \cmidrule(r){2-4} \cmidrule(r){5-7} 
		  Measurement  & - & 9.17 & 0.695 & - & 10.14 & 0.651 \\ 	
   		Momentum & 9 & 21.01 & 0.179 & 9 & 18.58 & 0.230 \\ 
            Momentum & 11 & 21.28 & 0.173 & 13 & 21.22 & 0.174 \\
            Momentum & 13 & 20.78 & 0.184 & 25 & 23.31 & 0.132 \\
            MLE & 11 & 20.45 & 0.191 & 25 & 22.65 & 0.148\\ 	
            WMC & 9,11,13 & 21.69 & 0.165 & 9,13,25 & 22.28 & 0.154 \\ 	
            WMC & 7,9,11,13,15 & 21.53 & 0.168 & 9,13,17,21,25 & 23.01 & 0.141 \\ 	
            \cmidrule(r){1-1} \cmidrule(r){2-4} \cmidrule(r){5-7}
			UNICORN (Ours)  & - & \textbf{28.28}(\textcolor{blue}{+6.75}) & \textbf{0.077}(\textcolor{blue}{-0.091}) & - & \textbf{25.72}(\textcolor{blue}{+2.11}) & \textbf{0.011}(\textcolor{blue}{-0.130}) \\ 			
			\bottomrule    
		\end{tabular}
	}
	\vspace{-0.5em}
	\label{tbl:main}
\end{table}

\begin{figure*}[!t]
\centering
\includegraphics[width = 0.95\linewidth]{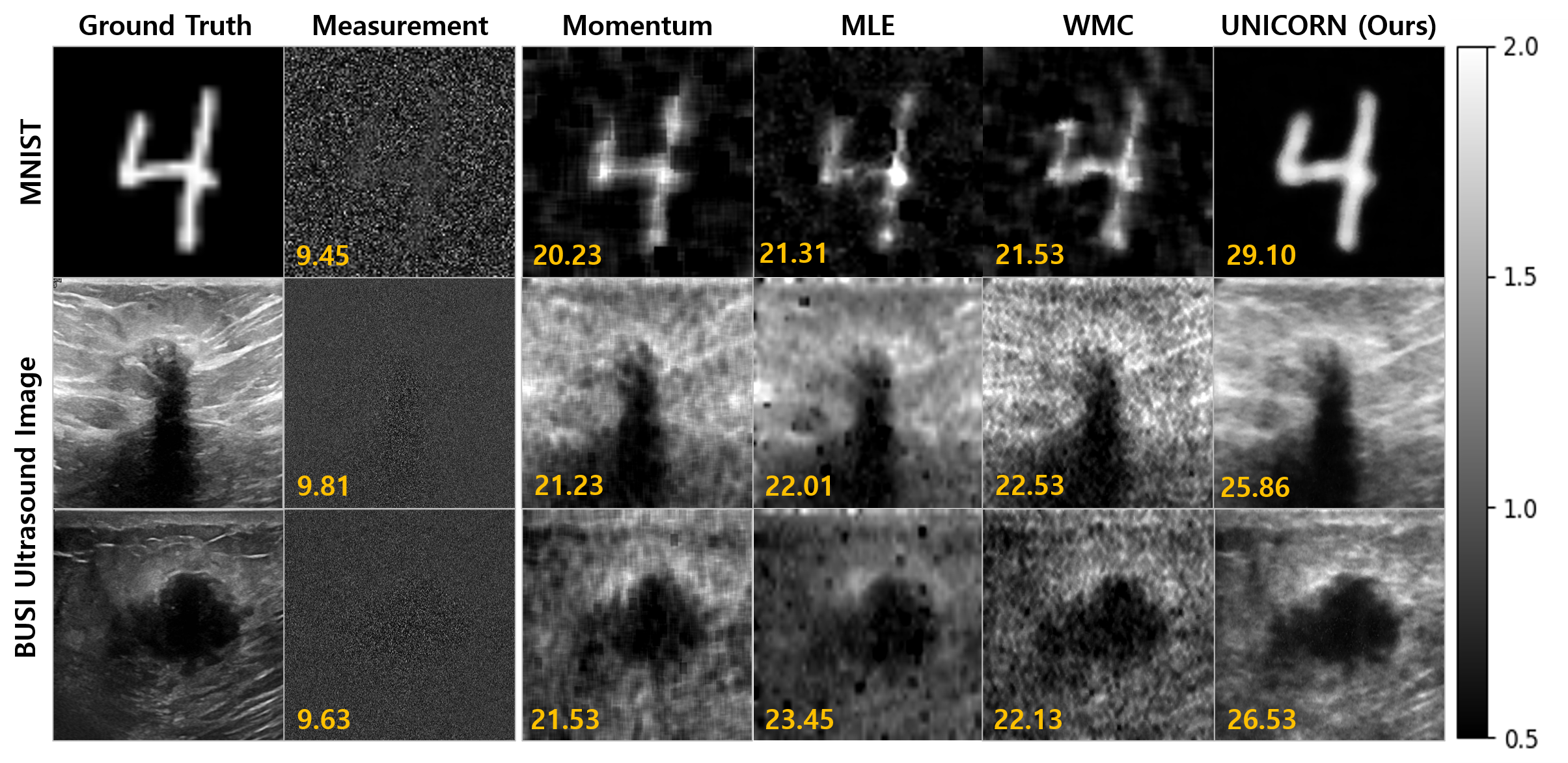}
\caption{Comparison of qualitative results across various methods: (Row 1) MNIST dataset, (Rows 2-3) BUSI Ultrasound Image dataset. Comparison against Momentum, MLE, WMC, and {UNICORN}. The yellow numbers indicate PSNR. To visualize the results, we normalize the images and convert them to grayscale.}
\label{fig}
\vskip -0.2in
\end{figure*}

\subsection{Results on Synthetic Simulation dataset}
In \cref{tbl:main}, we provide a comparison of quantitative performance using various methods on both the MNIST and BUSI datasets in terms of PSNR and RMSE. For the MNIST dataset, the PSNR of the measurement method is only 9.17 dB compared to the ground truth, indicating the highly ill-posed nature of the problem. 
Momentum-based approach results were sensitive to varied window sizes. Since MLE needs higher computational cost compared to the momentum-based approach, we select the best window size from momentum-based approaches for MLE and set the step size to half of the window size. 
Our proposed method, UNICORN, does not require any window size optimization but still achieves the highest PSNR of 28.28 dB and the lowest RMSE of 0.077, surpassing all other methods. UNICORN also outperforms the existing state-of-the-art method by a significant margin of +6.75 dB in PSNR. Similarly, on the BUSI dataset, UNICORN achieves superior performance with a PSNR of 25.71 dB and an RMSE of 0.089, representing a margin of +2.7 dB in PSNR.

In \cref{fig}, we compare the qualitative results of our proposed method against existing methods. We observe that the results obtained with Momentum-based approaches and the MLE approach are too blurred out and exhibit artifacts, such as the tile effect. While the WMC approach mitigates blurring to some extent, it introduces speckle noise artifacts. In contrast, our proposed method is capable of closely estimating Nagakami parameter mapping compared to the ground truth label in MNIST and provides superior resolution without significant artifacts in BUSI dataset.

\begin{figure*}[!t]
\centering
\includegraphics[width = 0.96\linewidth]{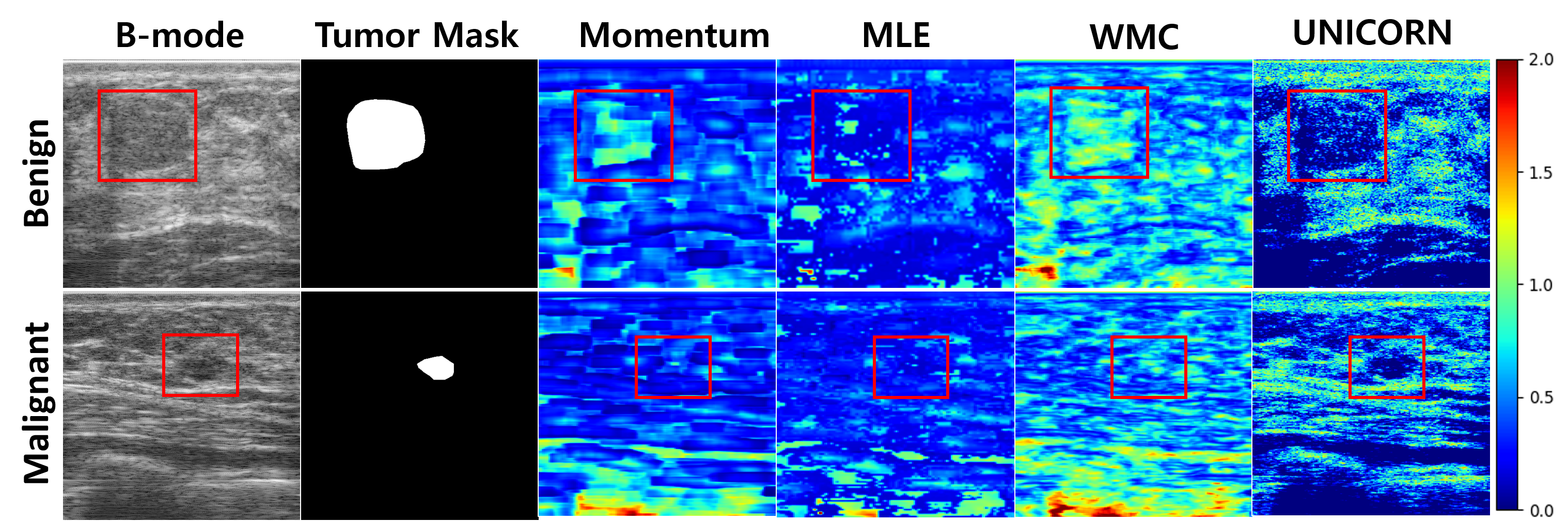}
\vspace{-0.5em}
\caption{Comparison of qualitative results across various methods: (Row 1) benign tumor, (Row 2) malignant tumor example. Comparison against Momentum, MLE, WMC, and {UNICORN}. The red line indicates the corresponding ROI of tumor.}
\label{fig:oasbud}
\vspace{-0.8em}
\end{figure*}

\begin{figure}[!t]
    \centering
    \begin{subfigure}{0.48\textwidth} 
        \centering
        \includegraphics[width=1\textwidth]{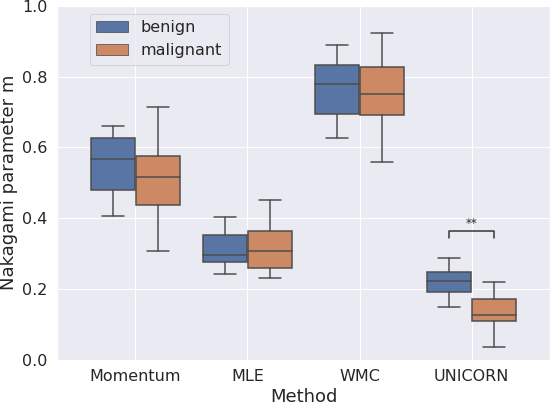}
        \label{fig:sub1}
        \caption{}
    \end{subfigure}\hfill
    \begin{subfigure}{0.45\textwidth} 
        \centering
        \includegraphics[width=1\textwidth]{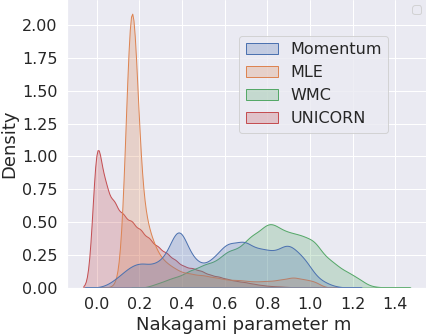}
        \caption{}
        \label{fig:sub2}
    \end{subfigure}
    \vspace{-1em}
    
    \caption{Comparison of Nakagami parameters within breast tumor ROIs. (a) box plots illustrating the Nakagami parameter distribution for benign and malignant tumors. (b) Histogram depicting the distribution of Nakagami parameters within benign tumors (refer to Figure \ref{fig:oasbud}).}
    \label{fig:anlaysis}
    \vspace{-1em}
\end{figure}

\subsection{Results on Real RF Envelope dataset}

Now, using RF envelop dataset from real US imaging,  we compare qualitative results with various methods as illustrated in Figure \ref{fig:oasbud}. The sliding window size for conventional methods was set to three times the pulse length. Three different window sizes from the pulse length to three times the pulse length were used to compute WMC. Our findings indicate that our proposed method yields significantly distinct features within the tumor by producing Nakagami parameters smaller than 1. Moreover, conventional approaches fail to capture the detailed contour of the tumor compared to B-mode images. In contrast, our proposed method can provide the tumor contour while preserving resolution. Furthermore, our method yields consistent results on both benign and malignant tumors. 

\cref{fig:anlaysis}(a) shows box plots of the Nakagami parameter for benign and malignant tumors compared to other approaches. In nearly all approaches, the average Nakagami parameter falls below 1, suggesting a pre-Rayleigh distribution. Our method exhibits the closest average value to 0 against the other approaches. Furthermore, \cref{fig:anlaysis}(b) depicts a histogram illustrating the Nakagami parameter of benign tumors across various approaches. While momentum-based methods encompass distributions ranging from pre-Rayleigh to post-Rayleigh, our approach aligns with the pre-Rayleigh distribution when compared to baseline methods.

\section{Discussion and Conclusion}

In this work, we introduced UNICORN, a novel framework for ultrasound Nakagami Imaging, addressing limitations of existing methods in visualizing tissue scattering of ultrasound waves. Our method incorporates the score function of a beamformed radiofrequency envelope and its signals, providing a closed-form estimator per pixel followed by low-pass filter adaptation. 
Traditional methods typically employ a sliding square window to compute the Nakagami parameter, focusing primarily on tissue transitions within the window. This approach, however, does not allow for visualization of the parameter at the individual pixel level. In stark contrast, our pixel-level estimator method is designed to capture the Nakagami parameter directly at each pixel, marking a clear differentiation from prior techniques. 
Accordingly, a key advantage of our approach is its ability to provide detailed contours of tumor tissue and distinct pre-Rayleigh distributions. This capability significantly enhances the differentiation between benign and malignant tumors, offering a more nuanced and accurate diagnostic tool compared to conventional baseline methods.
In our simulation experiments, UNICORN demonstrated superior performance over traditional methods, achieving a significant margin of improvement. 
 Furthermore,  by applying real ultrasound envelope data we demonstrate that our proposed approach can be extended to breast tumor detection. We believe that our framework holds promise for various applications in tumor diagnosis and fat fraction estimation, paving the way for advancements in ultrasound imaging techniques.

%
%
%
%


\clearpage
\newpage
\setcounter{page}{1}
\setcounter{section}{0}
\setcounter{figure}{0}
\setcounter{table}{0}
\appendix
\section{Proof of \cref{prop1}} \label{proof1}
\begin{proof}
For a given Nakagmi distribution ~\cref{pd:naka}, the log-likelihood function can be described as 
\begin{align}
    \log p_{R}(r) =-\log(\Gamma(m)) + m \log(m) - m\log(\Omega) + (2m-1)\log(r) -\frac{m}{\Omega}r^2 \label{log_liklihood}
\end{align}
Taking the derivative of \cref{log_liklihood} with respect to $r$, and setting it equal to zero, we have,    
\begin{align}
   \nabla_r \log p_{R}(r)= (2m-1) \frac{1}{r} -\frac{2m}{\Omega}r 
\end{align}
Accordingly,
\begin{align}
  m\left(\frac{2}{r} - \frac{2r}{\Omega}\right) &= \frac{1}{r} + \nabla_r \log p_{R}(r) 
\end{align}
Furthermore, we have,
\begin{align}
 \hat \Omega = \mathbb{E}[R^2] 
\end{align}
Therefore, we have,
\begin{align}
   \hat m &= \frac{\frac{1}{r} + \nabla_r \log p_{R}(r) }{\left(\frac{2}{r} - \frac{2r}{\hat \Omega}\right)}, \quad \text{where} \; \hat \Omega = \mathbb{E}[R^2] 
\end{align}
This concluded the proof.
\end{proof}

\section{Additional Visual Result}
\begin{figure*}[!h]
\centering
\includegraphics[width = 1\linewidth]{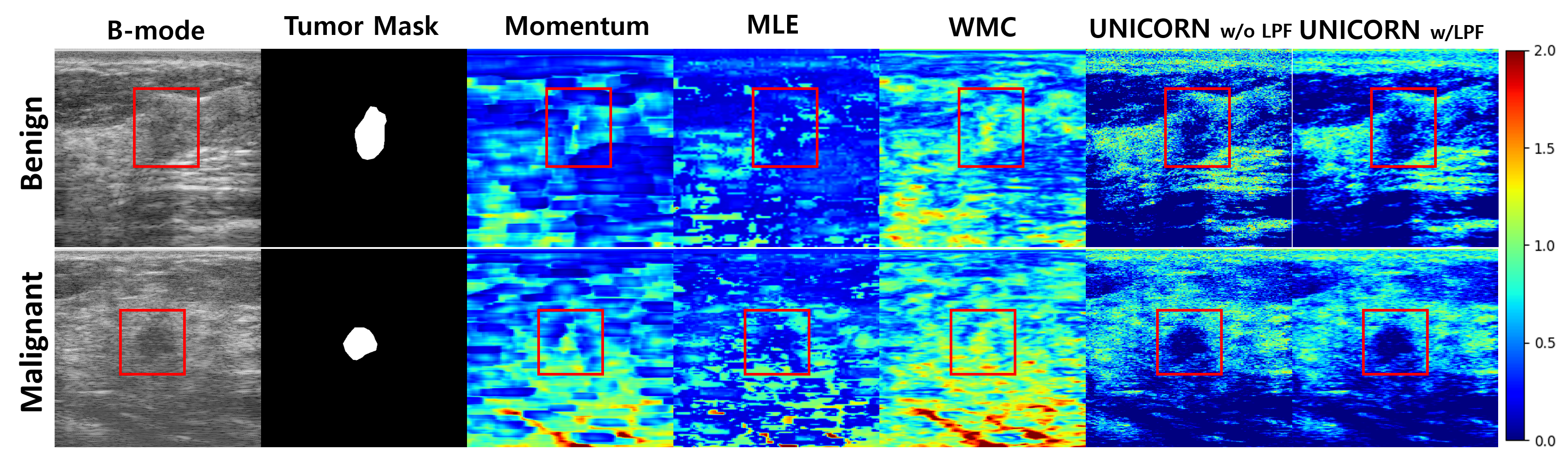}
\vspace{-0.5em}
\caption{Comparison of qualitative results across various methods: (Row 1) benign tumor, (Row 2) malignant tumor example. Comparison against Momentum, MLE, WMC, {UNICORN without low pass filters} and {UNICORN with a low pass filter}. The red line indicates the corresponding ROI of tumor.}
\label{fig:oasbud-sup}
\vspace{-0.8em}
\end{figure*}

\section{The Further detail of Implementation} \label{implement}
\begin{table}[!h] 
	\centering
	\caption{The Table representing implementation detail}
	\resizebox{0.95\linewidth}{!}{
		\begin{tabular}{cccc}
        \toprule
        Dataset & \multicolumn{1}{c}{MNIST} & \multicolumn{1}{c}{BUSI}  & \multicolumn{1}{c}{OASBUD }\\ 
        \cmidrule(r){1-1} \cmidrule(r){2-4} 
        model & modifeid U-Net & modifeid U-Net & modifeid U-Net \\ 
        batchsize & 512 & 16 & 16 \\
        learning rate & 2e-4 &  2e-4 &   2e-4\\
        optimizer & AdamW &  AdamW &  AdamW\\
        total epoch & 100 & 50 & 50 \\
        low-pass filter & Median filter & Median filter & Average filter \\
        $\sigma_a^{\min}$, $\sigma_a^{\max}$ &0.01, 0.1 & 0.03, 0.1 & 0.03, 0.1 \\
        image resolution & 28 $\times$ 28  & 500 $\times$ 500 & 1580 $\times$ 610  \\
        total data size & 50,000 & 780 & 100  \\
        division for dataset & train / test  & randomly select 80$\%$ / 20$\%$ & randomly select 80$\%$ / 20$\%$   \\
    \bottomrule    
    \end{tabular}
	}
	\vspace{-0.5em}
	\label{tbl:main}
\end{table}

\end{document}